%% file: 0_egpaper_for_review.tex
\documentclass[10pt,twocolumn,letterpaper]{article}
\pdfoutput=1 
\usepackage[accsupp]{axessibility}  
\usepackage{iccv}
\usepackage{times}
\usepackage{graphicx}
\usepackage{amsthm}
\usepackage{amssymb}
\usepackage{amsmath}
\usepackage{caption}
\usepackage{subcaption}
\usepackage{authblk}
\usepackage[normalem]{ulem}

\usepackage{mathtools}

\DeclarePairedDelimiter\floor{\lfloor}{\rfloor}

\newcommand*\samethanks[1][\value{footnote}]{\footnotemark[#1]}

\newtheoremstyle{bfnote}%
{}{}%
{\itshape}{}%
{\bfseries}{.}%
{ }%
{\thmname{#1}\thmnumber{ #2}\thmnote{ (#3)}}

\theoremstyle{bfnote}
\newtheorem{definition}{Definition}
\newtheorem{lemma}{Lemma}
\newtheorem{corollary}{Corollary}
\newtheorem{proof}{Proof}

\usepackage{hyperref}
\hypersetup{pagebackref=true,breaklinks=true,letterpaper=true,colorlinks,bookmarks=false}

\iccvfinalcopy 

\def\iccvPaperID{1997} 

\ificcvfinal\fi

\begin{document}

\title{How Shift Equivariance Impacts Metric Learning for Instance Segmentation}

\author{Josef Lorenz Rumberger$\thanks{equal contribution}$ $^{1,2}$,
Xiaoyan Yu$\samethanks$ $^{1,3}$,
Peter Hirsch$\samethanks$ $^{1,3}$,
Melanie Dohmen$\samethanks$ $^{1,2}$,
Vanessa~Emanuela~Guarino$\samethanks$ $^{1,3}$,
Ashkan Mokarian$^{1}$,
Lisa Mais$^{1}$,
Jan Funke$^{4}$,
Dagmar Kainmueller$^{1}$
\\
$^1$ Max-Delbrueck-Center for Molecular Medicine in the Helmholtz Association (MDC),\\
Berlin, Germany, {\tt\small {\{firstnames.lastname\}@mdc-berlin.de}} \\ 
$^2$ Charité University Medicine, Berlin, Germany\\
$^3$ Humboldt-Universität zu Berlin, Faculty of Mathematics and Natural Sciences, Berlin, Germany\\
$^4$ HHMI Janelia Research Campus, Ashburn, VA, USA
}


\maketitle
\ificcvfinal\thispagestyle{empty}\fi

\begin{abstract}
Metric learning has received conflicting assessments concerning its suitability for solving instance segmentation tasks. It has been dismissed as theoretically flawed due to the shift equivariance of the employed CNNs and their respective inability to distinguish same-looking objects.
Yet it has been shown to yield state of the art results for a variety of tasks, and practical issues have mainly been reported in the context of tile-and-stitch approaches, where discontinuities at tile boundaries have been observed. To date, neither of the reported issues have undergone thorough formal analysis. 
In our work, we contribute a comprehensive formal analysis of the shift equivariance properties of encoder-decoder-style CNNs, which yields a clear picture of what can and cannot be achieved with metric learning in the face of same-looking objects. 
In particular, we prove that a standard encoder-decoder network that takes $d$-dimensional images as input, with $l$ pooling layers and pooling factor $f$, has the capacity to distinguish at most $f^{dl}$ same-looking objects, and we show that this upper limit can be reached.
Furthermore, we show that to avoid discontinuities in a tile-and-stitch approach, assuming standard batch size 1, it is necessary to employ valid convolutions in combination with a training output window size strictly greater than $f^l$, while at test-time it is necessary to crop tiles to size $n\cdot f^l$ before stitching, with $n\geq 1$. 
We complement these theoretical findings by discussing a number of insightful special cases for which we show empirical results on synthetic and real data. 
\newline
Code:\url{https://github.com/Kainmueller-Lab/shift_equivariance_unet}
\end{abstract}
\input{1_intro}
\input{2_method}
\input{3_conclusion}
\newline
\textbf{Acknowledgements}. We wish to thank Anna Kreshuk and Fred Hamprecht for inspiring discussions. Funding: J.L.R.: German Research Foundation RTG 2424. J.L.R, M.D., P.H., L.M., A.M. and D.K.: Berlin Institute of Health and Max-Delbrueck-Center for Molecular Medicine in the Helmholtz Association (MDC). P.H.: MDC-NYU exchange program and HFSP grant RGP0021/2018-102. X.Y.: Helmholtz Einstein International Berlin Research School in Data Science. V.E.G.: German Research Foundation CRC 1404. J.F.: HHMI. P.H., L.M. and D.K. were supported by the HHMI Janelia Visiting Scientist Program.
{\small
\bibliographystyle{ieee_fullname}
\bibliography{egbib}
}
\end{document}


\title{Supplementary Material\\
How Shift Equivariance Impacts Metric Learning for Instance Segmentation}

\author{Josef Lorenz Rumberger$\thanks{equal contribution}$ $^{1,2}$,
Xiaoyan Yu$\samethanks$ $^{1,3}$,
Peter Hirsch$\samethanks$ $^{1,3}$,
Melanie Dohmen$\samethanks$ $^{1,2}$,
Vanessa~Emanuela~Guarino$\samethanks$ $^{1,3}$,
Ashkan Mokarian$^{1}$,
Lisa Mais$^{1}$,
Jan Funke$^{4}$,
Dagmar Kainmueller$^{1}$
\\
$^1$ Max-Delbrueck-Center for Molecular Medicine in the Helmholtz Association (MDC),\\
Berlin, Germany, {\tt\small {\{firstnames.lastname\}@mdc-berlin.de}} \\ 
$^2$ Charité University Medicine, Berlin, Germany\\
$^3$ Humboldt-Universität zu Berlin, Faculty of Mathematics and Natural Sciences, Berlin, Germany\\
$^4$ HHMI Janelia Research Campus, Ashburn, VA, USA
}

\maketitle
\ificcvfinal\thispagestyle{empty}\fi

\input{5_1_supplement}

{\small
\bibliographystyle{ieee_fullname}
\bibliography{egbib}
}

%% file: 1_intro.tex
\section{Introduction}
Metric learning is a popular proposal-free technique for instance segmentation that often yields state-of-the-art results, particularly in applications from the biomedical domain for which proposal-based techniques do not apply~\cite{merhof2019instance,de2017semantic,kulikov2020instance,lee2019learning,lee2017superhuman,novotny2018semi,rumberger2020probabilistic}. 
In discord with its empirical success, numerous works from the computer vision community have noted a theoretical deficiency of metric learning for instance segmentation, namely that same-looking objects cannot be distinguished by means of shift equivariant CNNs~\cite{liu2018intriguing,novotny2018semi}. 
Empirical attempts at tackling this apparent deficiency include leveraging  pixel coordinates or encodings of said as additional inputs or features~\cite{kulikov2020instance,neven2019instance,novotny2018semi,stringer2020cellpose}, or limiting the problem to distinguishing neighboring objects~\cite{merhof2019instance,kulikov2018instance}, while related theoretical work is limited to discussions of shift equivariance properties of individual CNN layers like pooling~ \cite{azulay2018deep,scherer2010evaluation,zhang2019making} and upsampling~\cite{odena2016deconvolution}.

What is thus lacking to date is a comprehensive formal analysis of the shift equivariance properties of the encoder-decoder style CNNs typically employed for metric-learning-based instance segmentation, as well as an assessment of respective implications concerning the capacity of said CNNs to distinguish same-looking objects. 
To this end, in this paper, we prove that an encoder-decoder-style CNN with $l$ pooling layers and pooling factor $f$ is periodic-$f^l$ shift equivariant, and in consequence has the capacity to distinguish at most $f^{d l}$ instances of identical appearance in $d$-dimensional input images.

Concerning practical issues, biomedical applications often deal with large 3d input images and thus apply CNNs for instance segmentation in a tile-and-stitch manner to cope with GPU memory constraints. Here, issues with discontinuities in predictions at output tile boundaries, which lead to false splits of objects, have been reported~\cite{lee2019learning,reina2020systematic}. However, again, a formal analysis of the causes is lacking to date. %
To this end, we show that the potential for discontinuities to arise is intricately tied to the shift equivariance properties of the employed CNNs. 
We focus on metric learning with discriminative loss as a showcase~\cite{de2017semantic}, because it facilitates theoretical insights via cleanly visible effects: Training for constant embeddings within individual instances conveniently entails that discontinuities in predictions manifest as jumps. 
Our respective theoretical analysis entails simple rules for designing CNNs that are necessary to avoid discontinuities when predictions are obtained in a tile-and-stitch manner. 

%% file: 2_method.tex
\section{Analysis of Shift Equivariance Properties}
\begin{figure*}[ht]
    \centering
    \includegraphics[width=0.8\textwidth]{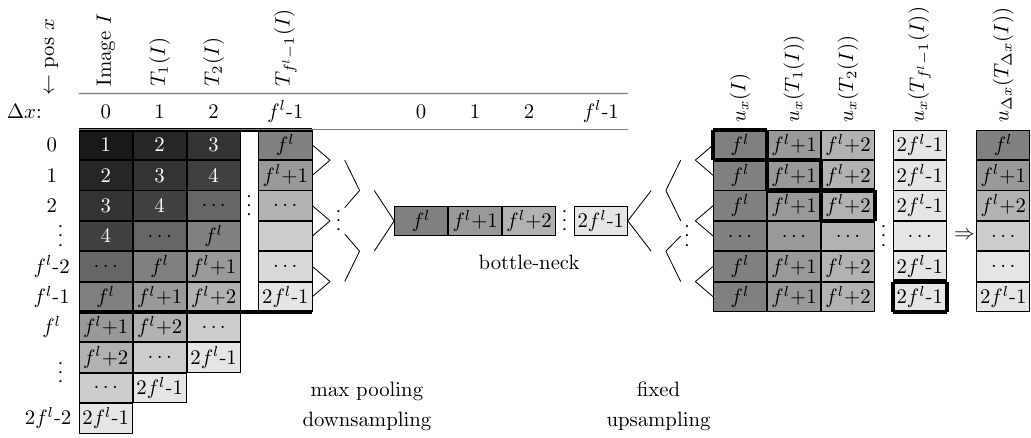}
    \caption{Illustration of a U-Net instance and a 1-dimensional image $I$ such that the functions $\textnormal{u}_{\Delta x}$ are relative-distinct for all $\Delta x$ with $0\leq\Delta x <f^l$.
    }
    \label{fig:fl_functions_in_fixed_unet}
\end{figure*}
We first define the broad family of U-Net-style encoder-decoder CNNs~\cite{ronneberger2015u} we consider, followed by a definition of \mbox{periodic-t} shift equivariance. Based on these prerequisites, we prove \mbox{periodic-$f^l$} shift equivariance of U-Nets. 
 
We consider CNNs consisting of $l$ downsampling and $l$ upsampling blocks. A downsampling block consists of a number of conv+nonlinearity layers, followed by max-pooling with downsampling factor (i.e.\ kernel size and stride) $f$. An upsampling block consists of a number of conv+nonlinearity layers, followed by upsampling by factor $f$, either via nearest-neighbor interpolation (\emph{fixed upsampling}) or via transposed convolution (\emph{learnt upsampling}). 
At each downsampling level of the U-Net, skip connections concatenate the output of the downsampling block before pooling to the input of the respective upsampling block after upsampling, except for the bottom level (also called \emph{bottleneck}). 
In the following, we refer to any achitecture of the above family as a \emph{U-Net}, and to a U-Net with specific weights as a \emph{U-Net instance}. A U-Net has the \emph{capacity} to have some property iff there exists an instance of that U-Net with said property. 
If not noted otherwise, we assume that a U-Net outputs all predictions for an image in one go. Sliding-window / tile-and-stitch mode will be discussed in Section~\ref{sec:tile-and-stitch}.
Furthermore, if not noted otherwise, we assume valid convolutions in all conv layers. Non-valid padding will be discussed in Section~\ref{sec:same-padding}.

Formally, a U-Net is a function U that maps a discrete, $d$-dimensional input image $I$ with resolution $X_1^{in}\times...\times X_d^{in}$ and $C^{in}$ channels to an output image with resolution $X_1^{out}\times ... \times X_d^{out}$ and $C^{out}$ channels:
\begin{equation}
\begin{aligned}
\textnormal{U}: \ &\mathbf{R}^{X_1^{in}\times..\times X_d^{in}\times C^{in}}\!\!\rightarrow\!\mathbf{R}^{X_1^{out}\times..\times X_d^{out} \times C^{out}} \\
&I \mapsto \textnormal{U}(I)\!=\!\left(\textnormal{u}_{\mathbf{x}}(I)\right)_{\mathbf{x}\in X_1^{out}\times..\times X_d^{out}},
\end{aligned}
\end{equation}
where
\begin{equation}
\begin{aligned}
\textnormal{u}_{\mathbf{x}}: \mathbf{R}^{X_1^{in}\times ...\times X_d^{in}\times C^{in}} \rightarrow \mathbf{R}^{C^{out}} \
I \mapsto \textnormal{u}_{\mathbf{x}}(I) = \textnormal{U}(I)(\mathbf{x}) \label{eq:u1}
\end{aligned}
\end{equation}
denotes the function that yields the output at output location $\mathbf{x}\in X_1^{out}\times ...\times X_d^{out}$.
Concerning functions $\textnormal{u}_{\mathbf{x}}$, two distinct notions of \emph{equality} can be defined:
\begin{definition}[Absolute and Relative Equality]
Two functions $\textnormal{u}_{\mathbf{x}_1}$, $\textnormal{u}_{\mathbf{x}_2}$ are \textbf{absolute-equal} iff $\forall I: \textnormal{u}_{\mathbf{x}_1}(I)=\textnormal{u}_{\mathbf{x_2}}(I)$,
and \textbf{absolute-distinct} otherwise. 
Two functions $\textnormal{u}_{\mathbf{x}_1}$,  $\textnormal{u}_{\mathbf{x}_2}$ are \textbf{relative-equal} iff 
$
\forall I: \textnormal{u}_{\mathbf{x}_1}(I)=\textnormal{u}_{\mathbf{x}_2}(T_{\mathbf{x}_2-\mathbf{x}_1}(I)),
$
with $T_{\Delta\mathbf{x}}(I(\mathbf{x})):= I(\mathbf{x}-\Delta\mathbf{x})$ denoting an image shift by $\Delta\mathbf{x}$.
Otherwise $\textnormal{u}_{\mathbf{x}_1}$ and $\textnormal{u}_{\mathbf{x}_2}$ are \textbf{relative-distinct}.
\end{definition}
\noindent We provide examples for absolute and relative equality of U-Net functions in Suppl.~Sec.~1. 
Following \cite{zhang2019making}, we define \emph{periodic-t shift equivariance} as follows:
\begin{definition}[Periodic-t Shift Equivariance]
A function $\textnormal{F}$ that maps an input image $I$ to an output image $\textnormal{F}(I)$ is \textbf{periodic-t shift equivariant} iff 
$
\textnormal{F}(T_{\Delta\textbf{x}}(I)) = T_{\Delta\textbf{x}}(\textnormal{F}(I))$  $\forall \Delta\textbf{x} \in \{ (z_1\cdot t, ..., z_d\cdot t) \ | \  z_i \in \mathbf{Z} \},
$
and $t$ is the smallest number for which this holds. 
\end{definition}
\begin{lemma}[Relative-distinct functions $\textnormal{u}$ of a U-Net]
 Every U-Net has the capacity to implement $f^{dl}$ relative-distinct functions $\textnormal{u}$, but not more. 
 \label{lem:num_func}
\end{lemma}
\begin{proof}
Part I: For any U-Net, we construct an instance and an image $I$ with unique outputs $u_{\textbf{x}-\Delta\textbf{x}}(T_{-\Delta\textbf{x}}(I))$ for all $\Delta\textbf{x} \in \{0,.., f^l-1\}^d$, proving that every U-Net has the capacity to implement \emph{at least} $f^{dl}$ relative-distinct functions.
Part~II: We prove that every U-Net is equivariant to image shifts $f^l$, and hence \emph{no} U-Net has the capacity to implement \emph{more than} $f^{dl}$ relative-distinct functions. 
~\\
\noindent\textbf{Proof part I: }
We construct a U-Net instance and an input image $I$ which yields $f^{dl}$ relative-distinct output function values, as described in the following. 
Fig.\ \ref{fig:fl_functions_in_fixed_unet} shows a sketch of our construction for $d=1$.
First, for any given U-Net, construct an instance $\textnormal{U}$ with fixed upsampling (i.e.\ all upsampling kernel weights set to 1), all convolutions set to identity, and ignore skip connections by setting respective convolution kernel entries to $0$. For a $d$-dimensional input image I, this U-Net instance yields outputs 
$$
\textnormal{u}_{\mathbf{x}}(I) = \max \{ I(\floor{\mathbf{x}/f^l}\cdot f^l + \Delta\mathbf{x}) \ | \ \Delta\mathbf{x}\in \{ 0,...,f^l-1 \} ^d  \}.
$$
Second, construct a single-channel image $I$ such that $I(\mathbf{x})$ is strictly increasing for increasing positions $\mathbf{x}$  w.r.t.\ an ordering of positions along diagonals first by the sum of their components and second by their components in increasing order 
(cf.\ \cite{Cantor1877,rosenberg2003}):
\begin{equation}
\mathbf{x}_i > \mathbf{x}_j\iff
\begin{cases}
\sum\limits_{k} x_i^{(k)} > \sum\limits_{k} x_j^{(k)}
&\text{if} \ \sum\limits_{k} x_i^{(k)} \neq \sum\limits_{k} x_j^{(k)}\\
x_i^{(k)}> x_j^{(k)}
&\text{if} \ \sum\limits_{k} x_i^{(k)} = \sum\limits_{k} x_j^{(k)}\\
&\text{ and }x_i^{(l)}= x_j^{(l)} \,\, \forall l<k
\end{cases}
\end{equation} 
For this image, the maximum intensity in any image pixel block of edge length $f^l$ is found at the maximum position $(f^l-1, ..., f^l-1)^d$. Consequently, as each distinct pixel block of edge length $f^l$ covers a unique maximum position, 
\begin{equation}
\begin{aligned}
\forall \mathbf{x} \ \forall \Delta\mathbf{x}_i 
    \neq \Delta\mathbf{x}_j \in &\{ 0,...,f^l-1 \} ^d: \\  \textnormal{u}_{\mathbf{x}-\Delta\mathbf{x}_i}(T_{-\Delta \mathbf{x}_i}(I)) & \neq \textnormal{u}_{\mathbf{x}-\Delta\mathbf{x}_j}(T_{-\Delta \mathbf{x}_j}(I)), 
\end{aligned}
\end{equation}
i.e.\ the constructed U-Net instance implements $f^{dl}$ relative-distinct functions $\textnormal{u}$. 
\textbf{Proof part II: } See Suppl.~Sec.~2. 
\end{proof} 
\begin{corollary}[Periodic-$f^l$ Shift Equivariance of U-Nets]
Every U-Net has the capacity to be periodic-$f^l$ shift equivariant. 
\end{corollary}
\begin{proof}
Directly follows from the proof of Lemma~\ref{lem:num_func}, which shows in Part I that every U-Net has the capacity to be non-equivariant to any shifts $<f^l$, and in Part II that every U-Net is shift equivariant to shifts $f^l$.
\end{proof}
\subsection{Tile-and-stitch mode}
\label{sec:tile-and-stitch}
\begin{figure*}[ht]
    \centering
    \begin{subfigure}[b]{0.98\textwidth}
    \includegraphics[width=\textwidth]{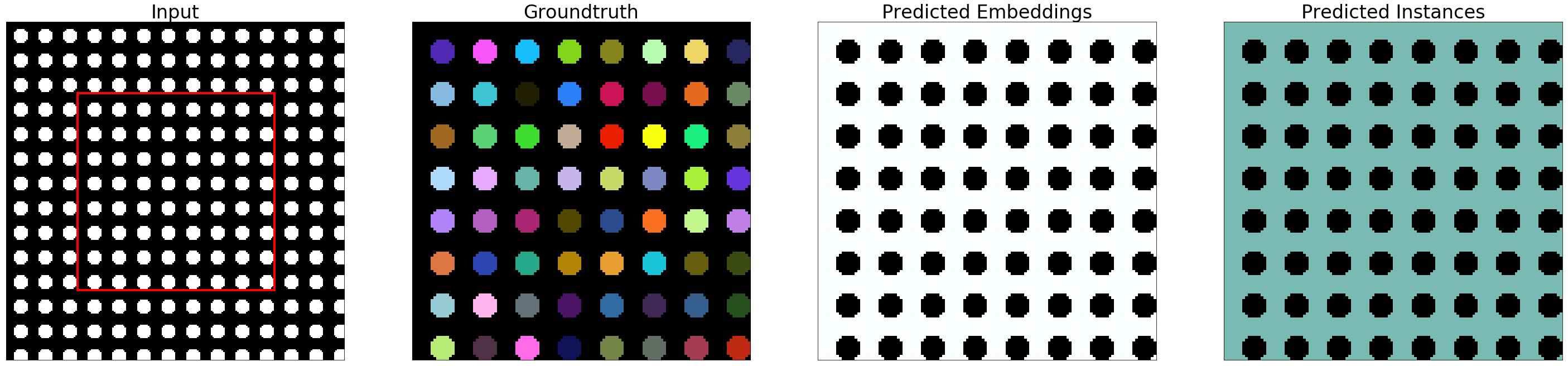}
    \caption{Object spacing $16$ pixels, i.e.\ a multiple of $f^l=8$. Learnt upsampling.}
    \label{fig:max_num_inst_a}
    \end{subfigure}
    \begin{subfigure}[b]{0.49\textwidth}
    \includegraphics[trim=1446 0 -31 0, clip,width=\textwidth]{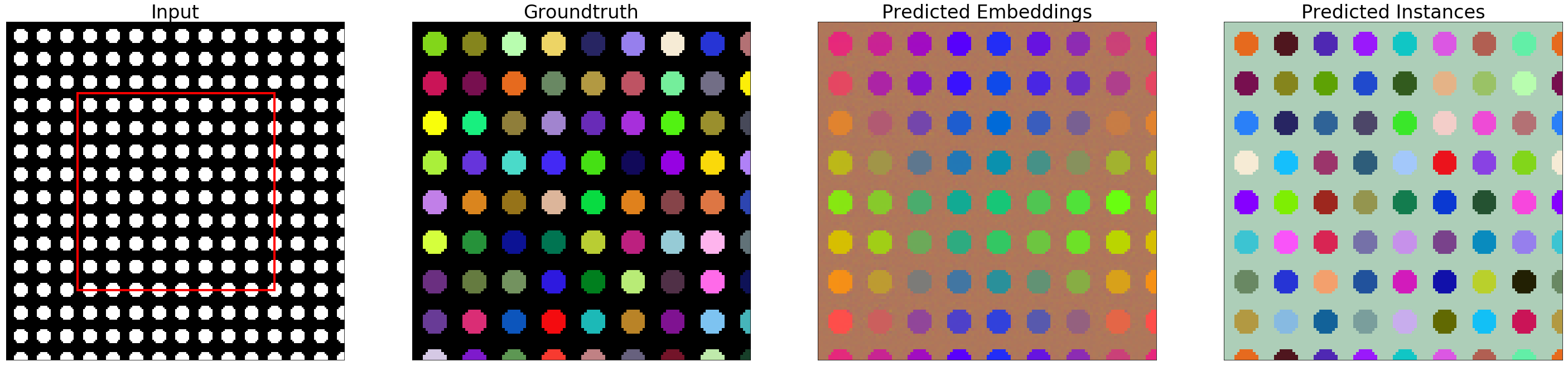}
    \caption{\centering Object spacing $15$ pixels, co-prime with $f^l=8$. Learnt upsampling.}
    \end{subfigure}
    \begin{subfigure}[b]{0.49\textwidth}
    \includegraphics[trim=1405 0 0 0, clip,width=\textwidth]{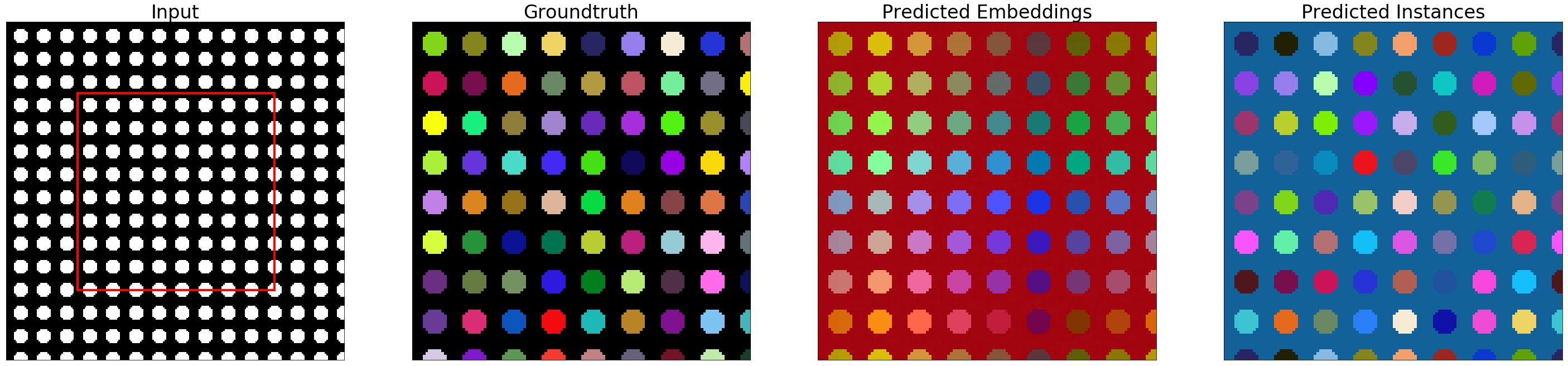}
    \caption{\centering Object spacing $15$ pixels, co-prime with $f^l=8$. Fixed upsampling.}
    \end{subfigure}
    \caption{A U-Net with $l$ pooling layers and pooling factor $f$ cannot distinguish any instances in an $f^l$-periodic $d$-dimensional image of same-looking instances (a). However, it can distinguish up to $f^{d\cdot l}$ instances in a $p$-periodic image of same-looking instances for $p$, $f^l$ co-prime (b,c). Showcase: $l=3, f=2, f^l=8, f^{dl}=64$. The red box in the input image (top left) shows the valid output window.  Analogous results can be achieved for the same object spacings and $l=4, f=2$ (not shown).
    }
    \label{fig:max_num_inst}
\end{figure*}
In practice, to deal with limited GPU memory, a U-Net is commonly trained on fixed-size input image tiles, yielding fixed-size output tiles. 
At test time, the output for a full input image is then obtained in a tile-and-stitch manner, where it is common to employ the same tile size as during training, yet larger tile sizes are sometimes employed as inference is less memory-demanding than training. 

Concerning shift equivariance in a tile-and-stitch approach with output tile size $w$ during inference, we get (1) periodic-$f^l$ shift equivariance within output tiles, and (2) trivially, periodic-$w$ shift equivariance across output tiles. Periodic-$f^l$ shift equivariance across the whole output only holds if $w$ is a multiple of $f^l$.
\subsection{Non-valid padding}
\label{sec:same-padding}
The concept of shift equivariance runs counter to the concept of non-valid padding, as the latter does not allow for ``clean" input image shifts: Shifting+padding, in general, changes the input image beyond the shift. 
As a notable consequence, e.g., zero padding renders a CNN with sufficiently large receptive field location-aware~\cite{kayhan2020translation,alsallakh2020mind}, thus eviscerating shift equivariance. See \cite{kayhan2020translation} for an in-depth discussion of zero-padding and other padding schemes. 
\section{Analysis of the Impact on Metric Learning for Instance Segmentation}
We assess implications of a U-Net's periodic-$f^l$ shift equivariance on the application of instance segmentation via metric learning with discriminative loss~\cite{de2017semantic}. The respective loss function has three terms, a pull-force that pulls pixel embeddings towards their respective instance centroid, a push force that pushes centroids apart, and a penalty on embedding vector lengths. Given predicted embeddings, instances are inferred by mean-shift clustering. For more details, see~\cite{de2017semantic}.
First, we assess how many ``same-looking" instances a U-Net trained with discriminative loss can distinguish. We call two instances ``same-looking" iff the image itself is invariant to shifting by the offset between instance center points. 
Second, we show the necessity to follow a concise set of simple rules to avoid inconsistencies in a tile-and-stitch approach. 

\subsection{Distinguishing Same-looking Instances (thus Avoiding False Merges)}
\begin{corollary}A U-Net has the capacity to distinguish at most $f^{dl}$ same-looking instances.
\label{cor:max-distinguished-instances}
\end{corollary}
\begin{proof}
Lemma~\ref{lem:num_func} entails that a U-Net can assign at most $f^{dl}$ different embeddings to a representative pixel of an object instance (say the ``central pixel"), namely when positioned at the $f^{dl}$ different relative locations w.r.t.\ its max-pooling regions. This holds true iff same-looking  instances are located at offsets $p$ with $p, f^l$ co-prime. 
\end{proof}
\noindent Whether a U-Net is also able to assign same embeddings to all pixels within any instance, thus yielding $f^{dl}$ correct segments, is up to its capacity and the success of training. 
\begin{corollary}A U-Net cannot distinguish same-looking instances located at offsets $n\cdot{f^l}, n\in\mathbf{N}$.
\label{cor:not-distinguished-instances}
\end{corollary}
\begin{proof}
Periodic-$f^l$ shift equivariance of the U-Net entails that it necessarily assigns same embeddings to pixels at same relative locations in the objects. \end{proof}
\noindent Fig.\ \ref{fig:max_num_inst} showcases Corollaries~\ref{cor:max-distinguished-instances} and~\ref{cor:not-distinguished-instances} on  images of periodically arranged disks, for which we trained U-Nets with discriminative loss to predict embeddings $\in \mathbf{R}^3$. In particular, it shows that the upper bound of separating $f^{dl}$ same-looking instances, as stated in Corollary~\ref{cor:max-distinguished-instances}, can be reached. 

\subsection{Avoiding False Split Errors in Tile-and-Stitch}\label{sec:Avoiding_False_Splits}
\begin{figure*}[t!]
    \centering
    \begin{subfigure}[b]{0.9\linewidth}
    \includegraphics[trim=32 262 398 80, clip,width=\linewidth]{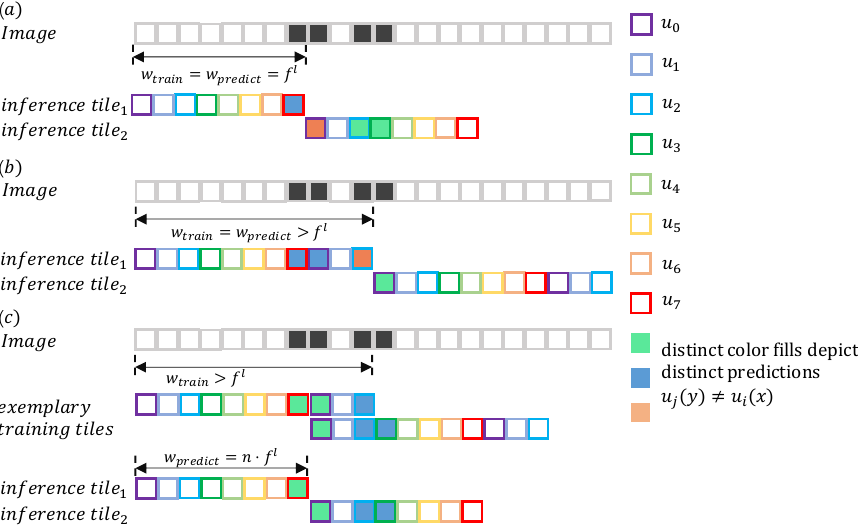}
    \end{subfigure}
    \caption{Stitching errors occur when two relative-distinct output functions $u_i$ and $u_j$ are adjacent to each other during inference, but not during training. Shown here is a 1-d sketch with $l=3$, $f=2$, i.e.\ with $f^l=8$ relative-distinct output functions $u_i$, and an exemplary input image containing two instances shown as black filled pixels (where each instance is two pixels wide). In (a) $u_7$ is adjacent to $u_0$ at the stitching boundary during inference, but during training they were never adjacent due to training output tile size $f^l$. This is fixed in (b) where the training output tile size is $> f^l$, but during inference $u_{2}$ is adjacent to $u_0$ which never occurred during training. In (c) the U-Net was trained as in (b) with training output tile size $> f^l$; During inference, however, output tiles are cropped to $n\cdot f^l$ to ensure that only functions that were adjacent to each other during training are adjacent at tile boundaries, thus allowing to overcome inconsistencies.
   }
    \label{fig:avoiding_false_splits}
\end{figure*}
\begin{figure*}[t!]
    \centering
    \begin{subfigure}[b]{0.98\textwidth}
    \includegraphics[width=\textwidth]{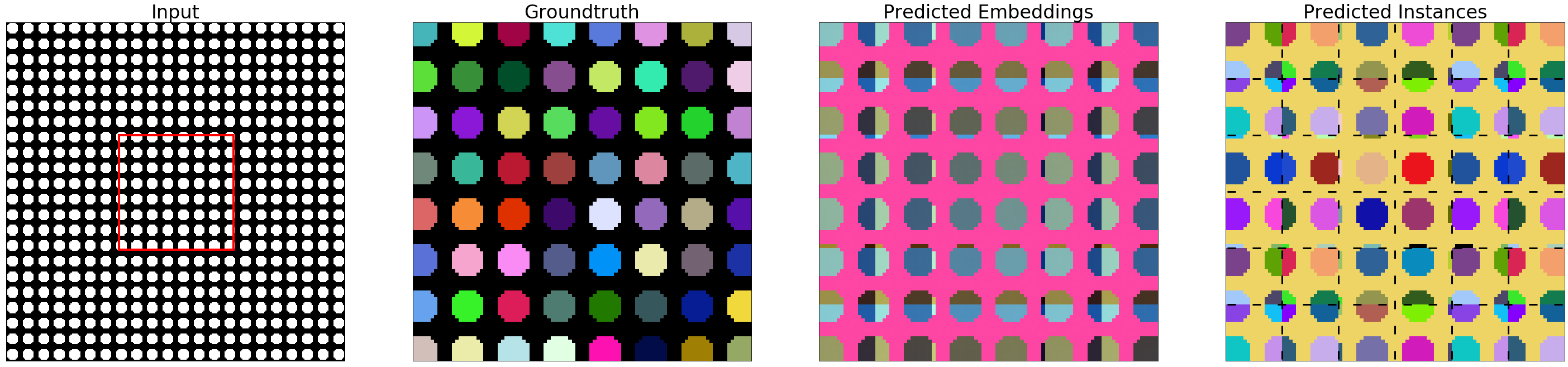}
    \caption{Output tile size $16 (=f^l)$}
    \label{fig:inconsistencies_a}
    \end{subfigure}
    \begin{subfigure}[b]{0.49\textwidth}
    \includegraphics[trim=1446 0 -31 0, clip,width=\textwidth]{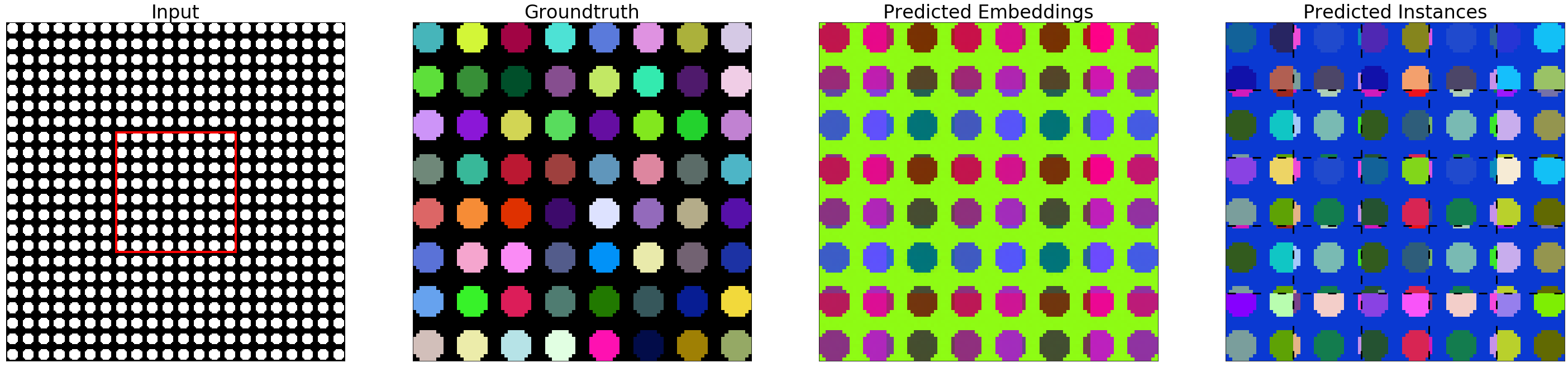}
    \centering
    \caption{\centering Output tile size $20 (>f^l)$, 
    not cropped before stitching}
    \label{fig:inconsistencies_b}
    \end{subfigure}
    \begin{subfigure}[b]{0.49\textwidth}
    \includegraphics[trim=1425 0 0 0, clip,width=\textwidth]{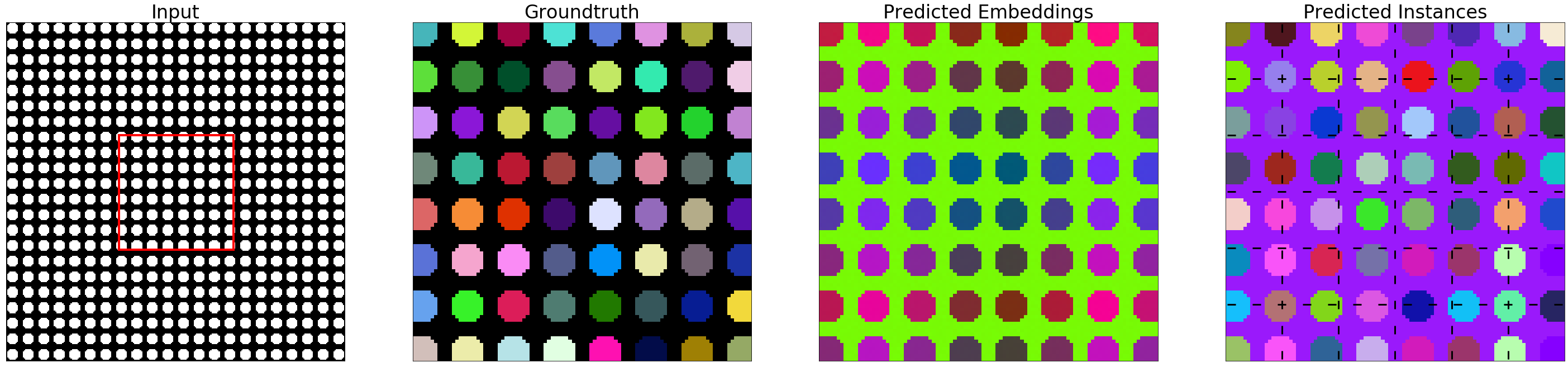}
    \centering
    \caption{\centering Output tile size 
    cropped to $16(=f^l)$ before stitching}
    \label{fig:inconsistencies_c}
    \end{subfigure}
    \centering
    \begin{subfigure}[b]{0.49\textwidth}
    \includegraphics[trim=0 0 1446 0, clip,width=\linewidth]{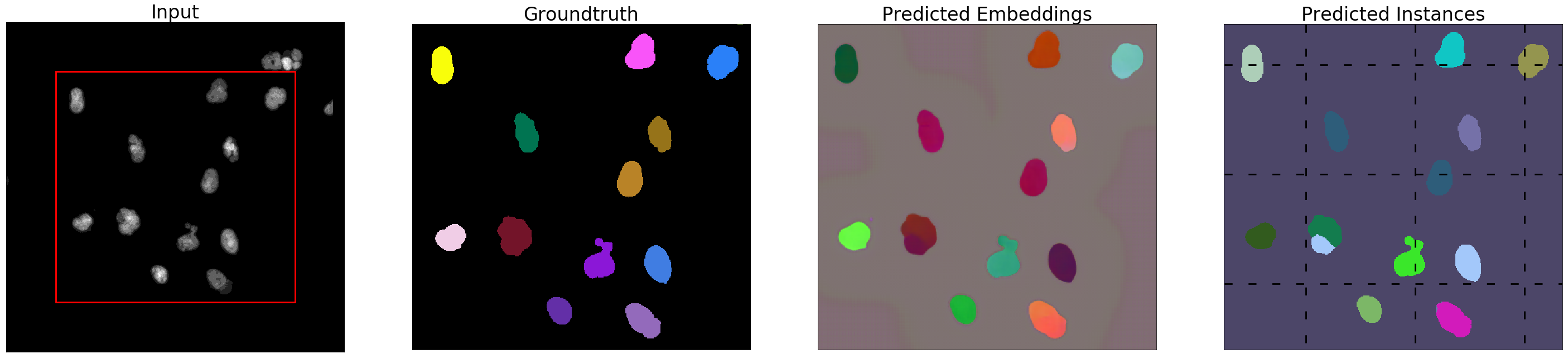}
    \centering
    \caption{\centering Sample \emph{mcf-z-stacks-03212011\_f22\_s} from the BBBC006 cell nuclei  dataset~\cite{ljosa2012annotated}}
    \label{fig:inconsistencies_d}
    \end{subfigure}
    \begin{subfigure}[b]{0.245\textwidth}
    \includegraphics[trim=2169 0 -31 0, clip,width=\linewidth]{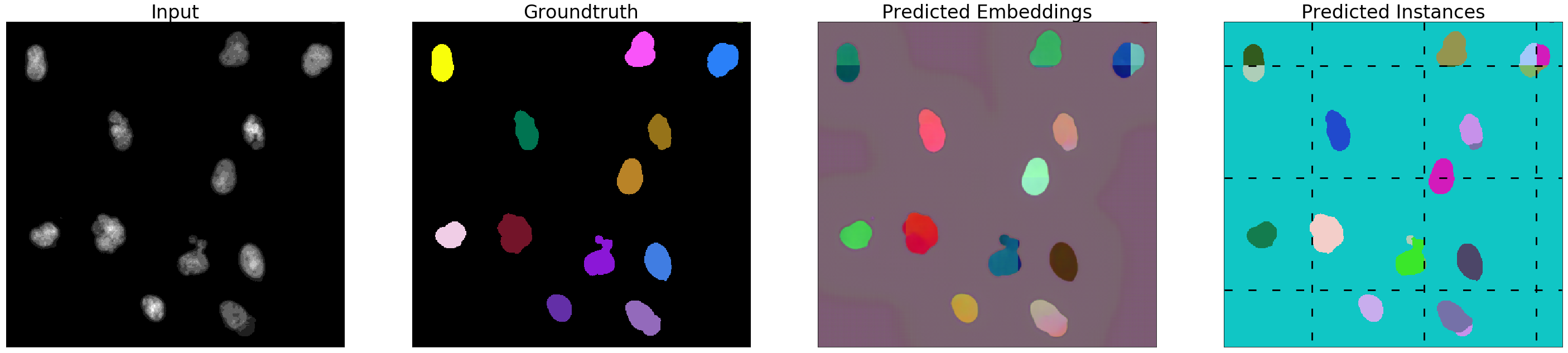}
    \centering
    \caption{\centering Output tile size $148 (>f^l)$, 
    not cropped before stitching}
    \label{fig:inconsistencies_e}
    \end{subfigure}
    \begin{subfigure}[b]{0.245\textwidth}
    \includegraphics[trim=2169 0 -31 0, clip,width=\linewidth]{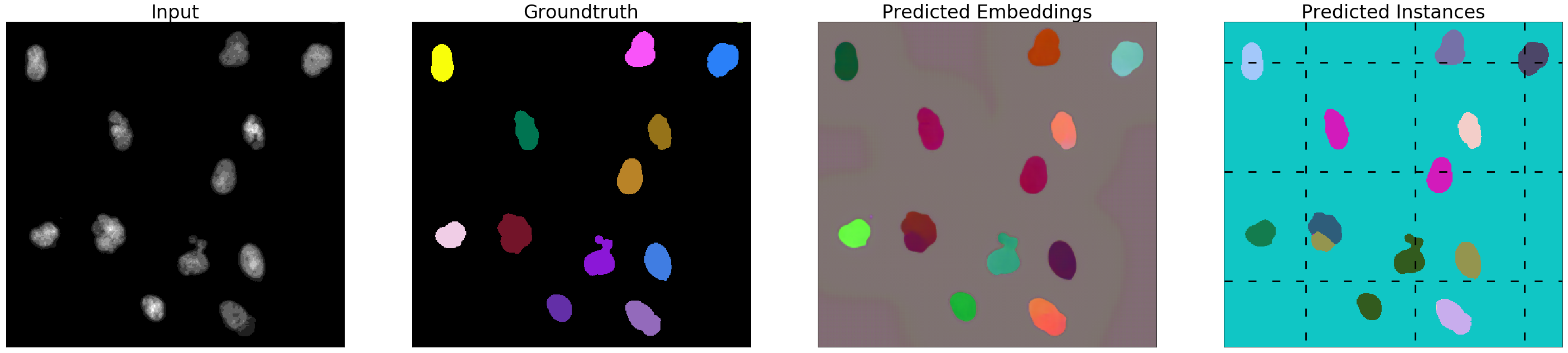}
    \centering
    \caption{\centering Output tile size 
    cropped to $144(=n \cdot f^l)$ before stitching}
    \label{fig:inconsistencies_f}
    \end{subfigure}
    \caption{Stitching issues, and how to fix them, for a U-Net with $l=4$ and $f=2$, and a $p$-periodic input image with $p, f^l$ co-prime. (a) Training with output window size $w=f^l$ yields inconsistencies at $f^l$-grid boundaries (black dashed lines) in larger outputs. To avoid inconsistencies, not only is it necessary to (b) train with $w>f^l$, which still yields inconsistencies at stitching boundaries when naively stitching $w$-sized tiles, but also to (c) crop tiles to size $n\cdot f^l$ before stitching, which solves the issue. (d-f) The same effect occurs on real data: (d) Excerpt from BBBC006 cell nuclei dataset~\cite{ljosa2012annotated}: Naive stitching (e) yields false split errors at tile boundaries, while correct stitching (f) fixes them. (See  Suppl.~Fig.~1 for resp.\ embeddings.) Note that slight differences between (e) and (f) \emph{within} tiles occur because predictions stem from different U-Net output functions.
    }
    \label{fig:inconsistencies}
\end{figure*}
In the following, we analyze the impact of output tile size on training with discriminative loss, as well as on inference in a tile-and-stitch manner. To this end, we assess which of the $f^{dl}$ potentially relative-distinct output functions of a U-Net contribute to the loss, and which pairs of functions that predict directly neighboring outputs in a stitched solution contribute to an instance's pull force loss term during training. 
Fig.~\ref{fig:avoiding_false_splits} exemplifies our analysis on a 1-d input image that contains a couple of two-pixel-wide instances. 

\noindent\textbf{Training output tile size $<f^l$: } In this case, some of the $f^{dl}$ output functions of the U-Net never contribute to the loss, i.e.\ they are not explicitly trained. In effect, they may yield nonsensical predictions when used during inference. 

\noindent\textbf{Training output tile size $=f^l$: } In this case, all output functions of the U-Net are considered during training in each batch. However, some pairs of functions that predict neighboring outputs during inference are never considered as neighbors during training. E.g.\ for $d=1$, $\textnormal{u}_0$ and $\textnormal{u}_{f^l-1}$ never predict directly neighboring embeddings during training, and hence never contribute to the pull force loss term as direct neighbors. They do, however, predict directly neighboring embeddings during inference, no matter if stitching $f^l$-sized output tiles or employing larger output tiles (potentially alleviating the need for stitching) during inference. Consequently, in this case, embeddings predicted at neighboring pixels at $f^l$-grid-boundaries may be inconsistent. 

\noindent\textbf{Training output tile size $>\!f^l$: } All possible direct neighborhoods of output functions are considered during training, given that inference output tile size is a multiple of $f^l$.

\noindent\textbf{Inference output tile size $\neq n\cdot f^l$: } Similar to the case of training output tile size $=f^l$, functions that predict neighboring outputs on two sides of a stitching boundary have never contributed to the same pull force term as neighbors during training (assuming batch size 1). Consequently, inconsistencies may occur at stitching boundaries. 

\noindent\textbf{Inference output tile size $= n\cdot f^l$: } Tile-and-stitch processing is guaranteed to not be causal for any inconsistencies, as formalized by the following Corollary: %
\begin{corollary}
If valid padding and output tiles of size $n\cdot f^l$ are employed, tile-and-stitch is equivalent to processing whole images at once.
\label{cor:tile_and_stitch_equiv}
\end{corollary}
%
\begin{proof}
This directly follows from  identical arrangements of respective output functions, 
namely arrangement into a regular grid of d-dimensional blocks of size $f^{dl}$.
\label{proof:tile_and_stitch_equiv}
\end{proof}

\noindent\textbf{Zero padding: }
\begin{figure*}[t!]
    \centering
    \begin{subfigure}[b]{0.98\textwidth}
    \includegraphics[width=\textwidth]{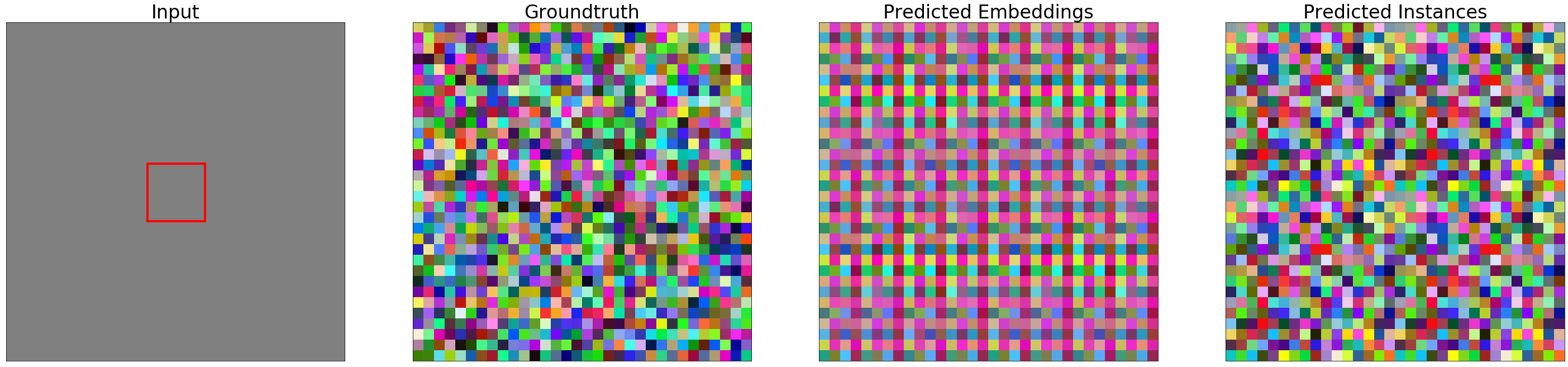}
    \end{subfigure}
    \caption{A U-Net with valid padding and learnt upsampling can learn to assign a unique ID to each pixel in an output window of size $f^l$, independent of the input image. This is not possible with nearest-neighbor upsampling. Showcase: $l=4, f=2$, input image $I\equiv 1$. Output: repeating pattern of $f^{dl}=256$ unique IDs.
   }
    \label{fig:pixel_IDs}
\end{figure*}
A U-Net with zero padding, training output window size $w$, and sufficiently large receptive field implements up to $w^d$ relative-distinct functions~\cite{kayhan2020translation}. Assuming batch size 1, this yields inconsistencies at stitching boundaries analogous to the valid-padding cases discussed above. Related work has attributed this effect to zero padding~\cite{huang2018tiling,reina2020systematic}, yet to our knowledge, mitigation has been limited to using larger tiles during inference~\cite{huang2018tiling,reina2020systematic}. Valid padding has been investigated as a potential remedy~\cite{reina2020systematic}, yet to no avail due to a lack of formal analysis. 

\noindent\textbf{Necessary rules to avoid inconsistencies at stitching boundaries: } 
Following from the above considerations, in general, to avoid inconsistencies in a tile-and-stitch approach, at training time, assuming standard batch size 1 and valid convolutions, it is necessary to train with output window size $>f^l$. 
Furthermore, at test time, it is necessary to crop output tiles of size $\neq n\cdot f^l$ to some $n\cdot f^l$ before stitching ($n\geq 1$).
Fig.~\ref{fig:inconsistencies_a}-\ref{fig:inconsistencies_c} showcases the necessity of following the rules on synthetic images of periodically arranged disks. 

For the case of zero padding, training with batch size $>1$ is necessary to avoid inconsistencies, where training output tiles in a batch have to be directly neighboring. 
Note, however, that batch size $>1$ is uncommon due to GPU memory limitations, and hence may entail further architectural changes to be feasible. 

\subsection{Location awareness}
\begin{corollary}
A U-Net with valid padding and learnt upsampling has the capacity to assign a unique ID to each pixel in an output window of size $f^{dl}$, independent of the specific input image.
\end{corollary}
\begin{proof}
Proof by construction: 
Set the first convolution to weights zero and bias 1. This yields a constant feature map. Set all other convolutions to identity. Thus, a feature map in the bottleneck layer will be constant. Ignore skip connections by setting respective convolution kernel entries to zero. 
Construct upsampling filter kernels $p_1 \dots p_l$ by filling them with non-repeating prime numbers. For this, $l \cdot f^d$ prime numbers are needed. Each of the $f^{dl}$ output functions $u_i$ of this U-Net instance yields a product over a unique set of $l$ distinct prime numbers. As the decomposition of any number into prime factors is unique,  respective outputs effectively assign a unique ID to each output pixel. 
\label{proof:unique_ID}
\end{proof}
Fig.~\ref{fig:pixel_IDs} showcases the level of location awareness that can be reached with a U-Net with valid padding and learnt upsampling, trained via metric learning with discriminative loss~\cite{de2017semantic} to segment pixels as individual instances given a constant input image. This confirms that a U-Net instance akin to the construction in Proof~\ref{proof:unique_ID} can be trained. A comparable effect of location awareness, albeit with conceptually different cause, has been described for zero-padding~\cite{kayhan2020translation}, which we showcase in Suppl.\ Fig.\ 2

Assigning unique IDs to pixels is yet another example of reaching the upper bound of distinguishing $f^{dl}$ instances (cf.\ Fig.\ \ref{fig:max_num_inst}), namely for the extreme case that each pixel in a constant input image forms an individual instance. However, this can only be achieved with learnt upsampling, or non-valid padding (cf.\ \cite{kayhan2020translation}). This is because for valid padding and fixed upsampling, a constant input image is always mapped to a constant output image. 

To our knowledge, our work is first to report location awareness given valid padding, thereby raising the question whether approaches that explicitly consider pixel locations or some other form of pixel IDs as additional inputs might be obsolete in case of valid padding and learnt upsampling. 
\section{Practical Impact}
We empirically assessed the practical impact of periodic-t shift equivariance on instance segmentation on synthetic images with added noise and deformations (Sec.~\ref{subsec:noise}), as well as on benchmark data (Sec.~\ref{subsec:benchmarks}).
\subsection{Noise and Small Deformations}
\label{subsec:noise}
A U-Net with $l$ levels and pooling factor $f$ fails to discriminate any instances in an infinite image of periodic-$f^l$ arranged objects (cf.~Fig.~\ref{fig:max_num_inst_a}). However we showcase in Suppl.~Figs.~3a, 3b that it may suffice to add slight Gaussian noise or small random elastic deformations to the input image to ``fix" the shift equivariance problem. Here, we generate noise or deformations randomly, on-the-fly per training step as well as at test time. Hence the observed effect is not due to over-fitting to a particular noisy/deformed image. 

However, note that neither noise nor elastic deformations do anything to fix the issue of inconsistencies in a tile-and-stitch approach if stitching is not performed according to the rules derived in Sec.~\ref{sec:Avoiding_False_Splits}, as illustrated in Suppl.~Figs.~3c,~3d.

\subsection{Quantitative Evaluation on Benchmark Data}
\label{subsec:benchmarks}
\noindent\textbf{Avoiding False Splits: } We assessed the practical impact of
correct tile-and-stitch on avoiding false split errors
on three cell nuclei segmentation datasets, namely BBBC006~\cite{ljosa2012annotated}, DSB2018~\cite{schmidt18_cell_detec_with_star_convex_polyg,Caicedo2019}, and nuclei3d~\cite{hirsch2020auxiliary,long20093d} (see Suppl.~Sec.~3 for details). %
We assessed AP~0.5, as well as false split- and false merge errors as defined in~\cite{caicedo2019evaluation}. 
We performed tile-and-stitch with a range of output window sizes. Correct tile-and-stitch, i.e.\ with output window size $n\cdot f^l$, drastically reduces false splits and increases AP~0.5 accordingly, as plotted in Fig.\ \ref{fig:split_avap_vs_crop}, and exemplified in  Fig.~\ref{fig:inconsistencies_d}-\ref{fig:inconsistencies_f} and Suppl.~Fig.~1.
\begin{figure*} 
    \centering
        \begin{subfigure}[b]{0.3\textwidth}
        \includegraphics[trim=0 0 0 0, clip,width=0.99\columnwidth]{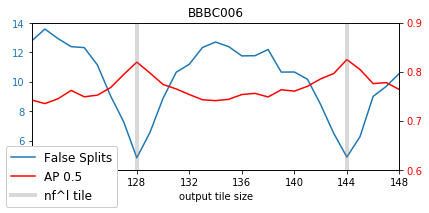}
    \end{subfigure}
    \hfill
      \begin{subfigure}[b]{0.3\textwidth}
        \includegraphics[trim=0 0 0 0, clip,width=0.99\columnwidth]{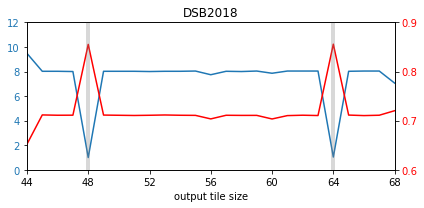}
    \end{subfigure}
    \hfill
      \begin{subfigure}[b]{0.3\textwidth}
        \includegraphics[trim=0 0 0 0, clip,width=0.99\columnwidth]{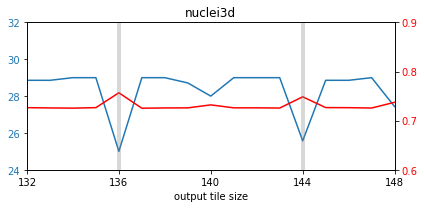}
    \end{subfigure}
    \caption{Output tile size vs.\ false splits and AP 0.5 on three cell nuclei datasets, covering 2d and 3d image data. 2d: Left:  BBBC006~\cite{ljosa2012annotated}, middle: the DSB2018~\cite{schmidt18_cell_detec_with_star_convex_polyg} subset of BBBC038v1~\cite{Caicedo2019}. 
    3d: Right: nuclei3d~\cite{hirsch2020auxiliary,long20093d}}
    \label{fig:split_avap_vs_crop}
\end{figure*}

\noindent\textbf{Distinguishing Instances: } We assessed the practical impact of periodic-$f^l$ shift equivariance on distinguishing instances on BBBC006
. If periodic-$f^l$ shift equivariance were of practical impact here, we would (1) expect to see false merge errors of \emph{distant}, i.e.\ non-touching, objects (whereas \emph{touching} objects may be merged for other, confounding reasons), and (2) we would expect distant false merges to occur for \emph{similar-looking} objects. On average, for 97.3 instances per test image, 3.8 false merge errors occur, and merged instances do not look similar by eye, as exemplified in Fig.~\ref{fig:distant-false-merge}.
This empirical study cannot prove that periodic-$f^l$ shift equivariance is not of  practical impact on distinguishing instances -- this would only follow if there were \emph{no} false merges, which is unlikely due to chance alone. However, it does suggest that the impact is negligible. 
\begin{figure}
    \centering
    \includegraphics[trim=20 20 20 20, clip,width=0.2\columnwidth]{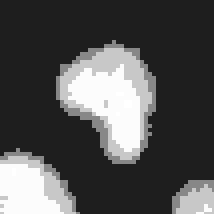}
    \includegraphics[trim=20 20 20 20, clip,width=0.2\columnwidth]{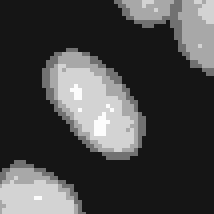}
    \caption{Exemplary distant-merged objects do not look alike in raw images (excerpts from BBBC006~\cite{ljosa2012annotated}), making periodic-$f^l$ shift equivariance an unlikely culprit.}
    \label{fig:distant-false-merge}
\end{figure}

%% file: 3_conclusion.tex
\section{Discussion and Conclusion}
Our work provides a formal analysis of the impact of shift equivariance properties of common encoder-decoder style CNNs on the task of metric learning for instance segmentation. Contrary to a range of works that have dismissed it as fundamentally flawed due to the assumed shift equivariance of CNNs, our theoretical analysis reveals the precise shift equivariance properties of U-Net style CNNs, from which follows that a U-Net with $l$ levels and downsampling factor $f$ is indeed able to distinguish up to $f^{dl}$ identical-looking (in terms of their respective receptive fields) instances in a $d$-dimensional image, given that object spacing is co-prime to $f^l$ in any dimension. 
In particular, our work refutes some findings of Novotny et al.~\cite{novotny2018semi} on similar synthetic imagery of periodically arranged discs (cf.\ their \href{https://openaccess.thecvf.com/content_ECCV_2018/papers/Samuel_Albanie_Semi-convolutional_Operators_for_ECCV_2018_paper.pdf}{\mbox{Fig.\ 3c}} in~\cite{novotny2018semi}): They attribute the observed ``near-random'', noise-like patterns within instances to the assumed ill-suitedness of metric learning for the task of instance segmentation, whereas our results on comparable data exhibit clean clusters in all cases (cf.~our Fig.~\ref{fig:max_num_inst}). As for differences in our model and theirs, they omit the push force in their ``simplified'' discriminative loss, while we employ it, thus avoiding that constant embeddings across all instances constitute a global optimum. Furthermore, they employ k-means clustering with k the correct number of instances, while we employ mean shift clustering, thus avoiding that clustering results are ill-defined in case of constant embeddings across instances (cf.~our Fig.~\ref{fig:max_num_inst_a}). 
Concerning the specific patterns within the disks in their \href{https://openaccess.thecvf.com/content_ECCV_2018/papers/Samuel_Albanie_Semi-convolutional_Operators_for_ECCV_2018_paper.pdf}{\mbox{Fig.\ 3c}}, a shift equivariant CNN would necessarily yield \emph{identical patterns for identical instances}. Instead, the figure shows multiple periodically alternating distinct patterns within instances, which violates their general assumption of shift equivariance, but is consistent with our theory, given that their truncated ResNet50 architecture is periodic-4 shift equivariant as it employs one max pooling layer with $f\!=\!2$ and one convolutional layer with stride~2 (where stride works analogous to pooling in terms of its effect on shift equivariance).

Beyond our formal analysis of shift equivariance properties, we show empirically on synthetic data that adding barely visible amounts of noise or elastic deformation can enable a U-Net to distinguish objects even at "unfortunate" object spacing $f^l$. Furthermore, we show on real data that, while distant objects are falsely merged sporadically, this cannot straightforwardly be attributed to shift equivariance, as we do not find respective merged instances to look similar by visual inspection. 

We deem of even greater impact to practitioners our theoretical analysis of inconsistencies that have been reported when performing metric learning with discriminative loss for instance segmentation in a tile-and-stitch approach due to large, GPU-memory-busting inputs. To this end, our theoretical analysis of shift equivariance allows us to derive a simple set of rules that necessarily have to be followed to avoid inconsistencies at stitching boundaries when performing inference on large data. While our impact analysis in this work is tailored to metric learning with discriminative loss, the same theory of shift equivariance yields similar implications for other pixel-wise prediction tasks for which tile-and-stitch issues with inconsistencies have been reported, like semantic segmentation (as e.g.\ studied empirically in~\cite{reina2020systematic}) or image registration. In particular, the proven equivalence between whole-image prediction and tile-and-stitch prediction with output tile size $n \cdot f^l$ holds independent of the specific training task.

%% file: supp/5_1_supplement.tex
\section{Examples for equality of U-Net functions u}
\label{app:abs-rel-equality}
\noindent\textbf{Example 1:} For a U-Net with identity convolutions, weights 0 for skip connections, and fixed upsampling, functions $\textnormal{u}$ that merely pass the value of a bottleneck pixel through to the output (cf. Fig.~1 in the main paper) are absolute-equal, yet relative-distinct.

\noindent\textbf{Example 2a:} For a U-Net with output tile size $w=1$ (i.e.\ a single output pixel per tile), employed in a sliding-window fashion (cf.\ Sec.~2.1 in the main paper), all functions $\textnormal{u}$ are relative-equal, yet absolute-distinct.

\noindent\textbf{Example 2b:} For a U-Net with pooling factor $f=1$ (i.e.\ no pooling and thus full shift equivariance), all functions $\textnormal{u}$ are relative-equal, yet absolute-distinct.

\section{Proof of Lemma~1, Part II}
\label{app:proof}
\renewcommand{\floor}[1]{\left\lfloor #1 \right\rfloor}
\renewcommand{\ceil}[1]{\left\lceil #1 \right\rceil}
\newcommand{\xycomment}[1]{\color{olive} #1 \color{black}}
An operator $F$ operating on images $I$ is \emph{shift equivariant to image shifts t} iff shifting any input image $I$ by $t$ causes an equal or proportional shift $t^{\prime}$ in the function's output, i.e.\ 
$
\forall I \ \forall x: \ T_{t^{\prime}}(F(I))(x)=F(T_t(I))(x)
$. 
In case $t=t^{\prime}$, we call $F$ shift equivariant to input shifts $t$. In case $t\neq t^{\prime}$, we call $F$ shift equivariant to input shifts $t$ at output shifts $t^{\prime}$. 
In the following, we prove that any U-Net with $l$ pooling layers and pooling factor $f$ is shift equivariant to image shifts $f^l t, \ t \in \mathbf{Z}$. Without loss of generality, we consider one-dimensional, one-channel input images, and one-channel feature maps throughout. 
A U-Net is composed of an encoder path and a decoder path:

\newcommand{\EMP}{{EM\! P}}
\newcommand{\DUP}{{DU\! P}}
\noindent\textbf{Encoder Path. }
An encoder block is composed of a number of conv+ReLU layers. We refer to the function implemented by the i-th encoder block as $E_i$. The output of $E_i$ is passed through a max pooling layer, referred to as ${M\!P}_i$. We refer to the composition of $E_i$ and ${M\!P}_i$ as $\EMP_i$. 
The convolution operator $F_{conv}$ (stride=1) is commonly defined as $F_{conv}(g)(x)=(g\star h)(x)=\sum_{m \in \mathbf{Z}}g(m)h(m-x)$ (see e.g.\ \cite{cohen2016group}), and well-known and easily shown to be shift equivariant to any shifts $t$: 
\begin{equation}
\begin{aligned}
  F_{conv}(T_t(g))(x)
&=\sum_{m \in \mathbf{Z}}g(m-t)h(m-x) \\
& =\sum_{m^\prime \in \mathbf{Z}}g(m^\prime)h(m^\prime+t-x) \\
&=\sum_{m^\prime \in \mathbf{Z}}g(m^\prime)h(m^\prime-(x-t)) \\
& =T_t( F_{conv}(g))(x)
\end{aligned}
\end{equation}
The same holds for the ReLU operator:
$ ReLU(T_t(g))(x) = max(0, T_t(g)(x)) = max(0,g(x-t)) 
= T_t(ReLU(g))(x).
$
Compositions of shift equivariant operators are also shift equivariant, hence $E_i$ is shift equivariant to any input shifts $t$. 
To analyze ${M\!P}_i$, we break it down into the max pooling operator
$
    \psi_f(g)(x) := max\{g(x+i) \ | \ i \in \{ 0, ..., f-1 \} \}
$
with kernel size $f$, and the sub-sampling operator $s_f(g)(x):=g(fx)$ at stride $f$. 
Analogous to ReLU, $\psi_f$ is shift equivariant to any shift $t$, 
\begin{equation}
\begin{aligned}
\psi_f(T_t(g))(x) &= max\{T_t(g)(x+i) \ | \ i \in \{ 0, ..., f-1 \} \} \\
 &= max\{g(x+i-t) \ | \ i \in \{ 0, ..., f-1 \} \} \\
 &= T_t(\psi_f(g))(x),
\end{aligned}
\end{equation}
while $T_t(s_f(g)) (x) = s_f(T_{ft}(g)) (x)$, i.e.\ $s_f$ is shift equivariant to input shifts $ft$ at proportional shifts $t$ of the output. 
With this we get
\begin{equation}
\begin{aligned}
\EMP_i(T_{ft}(g))(x) &=  s_f (\psi_f (E_i(T_{ft}(g))) ) (x) \\
& = s_f (T_{ft} (\psi_f (E_i (g) ) ) )  (x) \\
&= T_t (s_f(\psi_f (E_i (g)))) (x) \\
&= T_t(\EMP_i(g))(x),
\end{aligned}
\end{equation}
i.e., an encoder block with max pooling with downsampling factor $f$ is shift equivariant to input shifts $ft$ at proportional output shifts $t$. 
Overall, a general encoder path $E$ employs $l$ downsampling operations with factors $f_1,...,f_l$. Shift equivariance proportionality factors multiply when composing operators, hence the encoder path is shift equivariant to input shifts $t\prod_1^l f_l$ at output shifts $t$. In the family of U-Nets we consider, $f_i=f_j$ for all $i,j$, yielding shift equivariance to input shifts $t f^l$ at output shifts $t$:
$
  E(T_{tf^l}(I))(x) = T_t(E(I))(x).  
$

\noindent\textbf{Decoder Path. }
The decoder path is composed of decoder blocks $D_i$, whose output is passed through respective upsampling layers, referred to as ${U\! P}_i$.
A decoder block has the same form as an encoder block, i.e.\ it consists of a number of conv+ReLU layers, and is thus shift equivariant. 
We refer to the composition of $D_i$ and ${U\! P}_i$ as $\DUP_i$. Upsampling is either \emph{learnt}, i.e.\ performed via up-convolution with trainable kernel function $p(x)$, with kernel size = stride = $f$ (also called upsampling factor), or performed via nearest neighbor interpolation. We treat both in one go, as the latter is a special case of the former with fixed kernel function $p(x)\equiv 1$.
We can express the up-convolution operator ${U\! P}_i$
with upsampling factor $f$ as ${U\! P}_i(g)(x)=(g \ast p)(x)=\sum_{m \in \mathbf{Z}}g(m)p(x-fm)$.
Concerning its shift equivariance, 
\begin{equation}
\begin{aligned}
    {U\! P}_i \left(T_t(g)\right)(x) &=\sum_{m \in \mathbf{Z}}g(m-t)p(x-fm) \\
    &=\sum_{m^\prime \in \mathbf{Z}}g(m^\prime)p(x-f(m^\prime+t)) \\
    &=\sum_{m^\prime \in \mathbf{Z}}g(m^\prime)p((x-ft)-d m^\prime) \\
    &=T_{tf}\left({U\! P}_i \left(g \right)\right)(x),
\end{aligned}
\end{equation}
i.e., upsampling with factor $f$ is shift equivariant to input shifts $t$ at output shifts $ft$. 
Thus a decoder block with subsequent upsampling layer, $\DUP_i$, is also shift equivariant to input shifts $t$ at output shifts $ft$: 
$
    \DUP_i(T_t(g))(x) = T_{ft}(\DUP_i(g))(x).
$
Concerning the input to $\DUP_i$, at the bottleneck level $i=l$, this is the output of $\EMP_l$. Concerning shift equivariance of their composition $U_l:= \DUP_l\circ \EMP_l$, assuming equal down- and upsampling factors $f$, we get
\begin{equation}
\begin{aligned}
   U_l(T_{ft}(g))(x) &= \DUP_l(\EMP_l(T_{ft}(g)))(x) \\
   &= \DUP_l(T_t(\EMP_l(g)))(x) \\
   &= T_{ft}(\DUP_l(\EMP_l(g)))(x) \\
   &= T_{ft}(U_l(g))(x),
\end{aligned}
\end{equation}
i.e., $U_l$ is shift invariant to shifts $ft$.
For $i<l$, the input to $\DUP_i$ is a multi-channel image formed by concatenating the output of $\DUP_{i+1}$ and the output of encoder block $E_{i+1}$.
Refering to the composition of all U-Net blocks up to a block $B$ as $\tilde{B}$, we can write the input to $\DUP_i$ as
$(T_{\Delta x_{i+1}}(\tilde{E}_{i+1}(I)),\tilde{\DUP}_{i+1}(I))$.
Here, $\Delta x_{i+1}$ is the shift required to centrally align $\tilde{E}_{i+1}(I)$ and $\tilde{\DUP}_{i+1}(I)$, as $\tilde{\DUP}_{i+1}(I)$ is of size smaller or equal than $\tilde{E}_{i+1}(I)$. All $\Delta x_i$ are fixed for a given architecture, and hence image concatenation is shift equivariant: 
$
    \left( T_{\Delta x_i}(T_t(g)) , \,\, T_t(q) \right)(x) 
    =T_t\left( \left( T_{\Delta x_i}(g), \, \, q \right) \right) (x).
$
Consequently, just like $\DUP_l$, for $i<l$, $\DUP_i$ is shift equivariant to input shifts $t$ at output shifts $ft$. 
Furthermore, as proportional shift equivariance factors multiply when composing respective operators, analogous to $U_l:=\DUP_l\circ \EMP_l$, we get that the composition of all blocks from $\EMP_i$ to $\DUP_i$,  $U_i:=\DUP_i \circ  U_{i+1} \circ \EMP_i$, is shift equivariant to shifts $f^{l-i+1}t$. In particular, $U_1$ is shift equivariant to shifts $f^l t$. 

To yield the full U-Net function $U$, the outputs of $U_1$ and $E_1$ are concatenated into a multi-channel image, which is passed through a final shift invariant decoder block
$D_0$, $U:= D_0((T_{\Delta x_1}{E}_1(I),U_1(I)))$
Thus $U$ is equally shift equivariant as $U_1$, i.e., the U-Net is shift equivariant to shifts $f^l t$. $\qedsymbol$

\section{Quantitative Evaluation on Benchmark Data}
\label{sup:real_data}
\noindent\textbf{BBBC006:} The dataset is split into 691 training and 77 test images. Images contain on average 97 instances. We trained a U-Net with $f\!=\!2, l\!=\!4$, 16-d embeddings and discriminative loss with training tile size $148$ and used two-fold cross-validation on the test data to tune hyperparameters.\\
\textbf{DSB2018:} The dataset is split into 380 training, 67 validation and 50 test images. Images contain on average 49 instances. We trained a U-Net with $f\!=\!2, l\!=\!4$, 16-d embeddings and discriminative loss with training tile size $68$. Aside from the loss the same setup as in~\cite{hirsch20_patch_instan_segmen} is used.\\
\textbf{nuclei3d:} The dataset is split into 18 training, 3 validation and 7 test volumes. Volumes contain on average 537 instances. We trained a U-Net with $f\!=\!2, l\!=\!3$, 16-d embeddings and discriminative loss with training tile size $148$. Aside from the loss the same setup as in ~\cite{hirsch2020auxiliary} is used.

\begin{figure*}[h!]
    \centering
    \begin{subfigure}[b]{0.24\textwidth}
    \includegraphics[trim=1446 0 670 0, clip,width=\linewidth]{figures/irl_wrong_crop_cropped.png}
    \centering
    \caption{\centering Output tile size $148 (>f^l)$, 
    not cropped before stitching}
    \end{subfigure}
    \begin{subfigure}[b]{0.24\textwidth}
    \includegraphics[trim=1446 0 670 0, clip,width=\linewidth]{figures/irl_right_crop_cropped.png}
    \centering
    \caption{\centering Output tile size 
    cropped to $144(=n \cdot f^l)$ before stitching}
    \end{subfigure}
    \begin{subfigure}[b]{0.48\textwidth}
    \includegraphics[trim=1094 0 -31 0, clip,width=\linewidth]{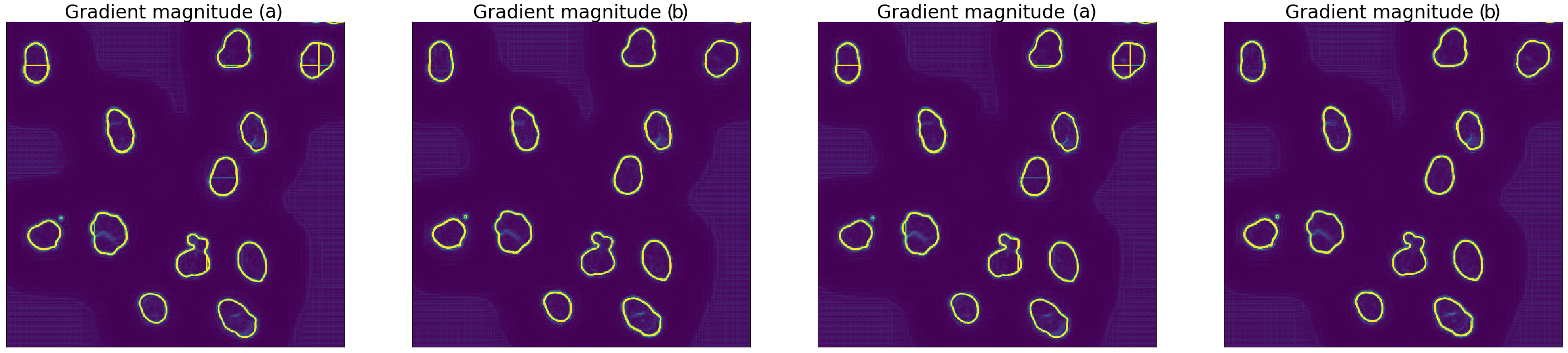}
    \centering
    \caption{\centering Gradient magnitude of predicted embeddings 
    \hspace{\linewidth} 
    in (a) and (b)}
    \end{subfigure}
    \caption{Predicted embeddings of a U-Net (a) naively stitched without cropping and (b) cropped to $n\cdot f^{l}$ before stitching. Inconsistencies at the stitching boundaries are clearly visible in the (c) gradient magnitudes of the embeddings.}
    \label{suppfig:inconsistencies}
\end{figure*}

\section{Zero padding leads to location awareness}
\label{sup:zero_padding}
\begin{figure*}[h!]
    \centering
    \begin{subfigure}[htpb]{0.45\linewidth}
    \includegraphics[trim=1446 0 -31 0, clip,width=\linewidth]{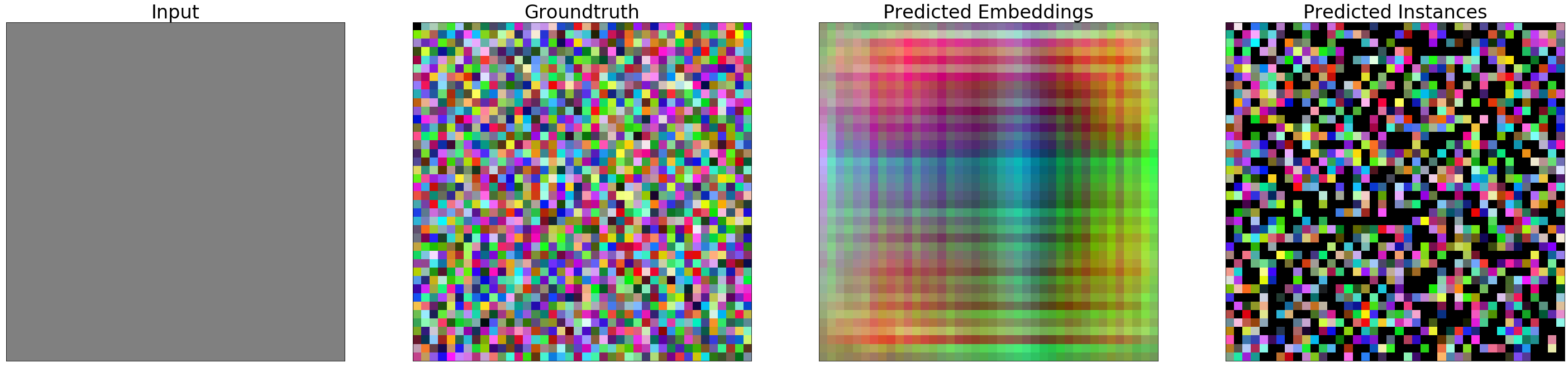}
    \caption{\centering Receptive field large enough for full location awareness at given input image size} 
    \end{subfigure}
    \hfill
    \begin{subfigure}[htbp]{0.45\linewidth}
    \includegraphics[trim=1446 0 -31 0, clip,width=\linewidth]{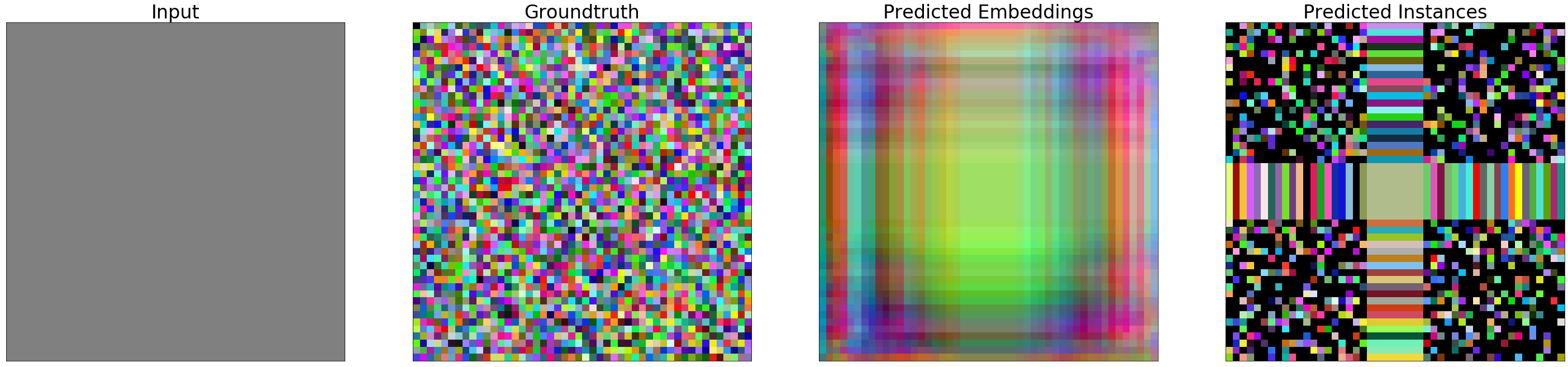}
    \caption{\centering Receptive field too small for full location awareness at given input image size} 
    \end{subfigure}
    \caption{A U-Net with zero-padding yields location awareness also with nearest-neighbor upsampling, if (a) each output pixel has a unique receptive field that reaches the image boundary. Otherwise, for a constant input image, (b) some pixels will necessarily receive equal outputs. Showcase: $l=2$, $f=2$, input image $I\equiv 1$.
    }
    \label{fig:pixel_IDs_zero_padding}
\end{figure*}

\clearpage
\section{Practical Impact of Noise and Small Deformations}
\label{sup:noise_and_deformations}
\begin{figure*}[h!]
    \begin{subfigure}[b]{0.98\textwidth}
    \includegraphics[width=\textwidth]{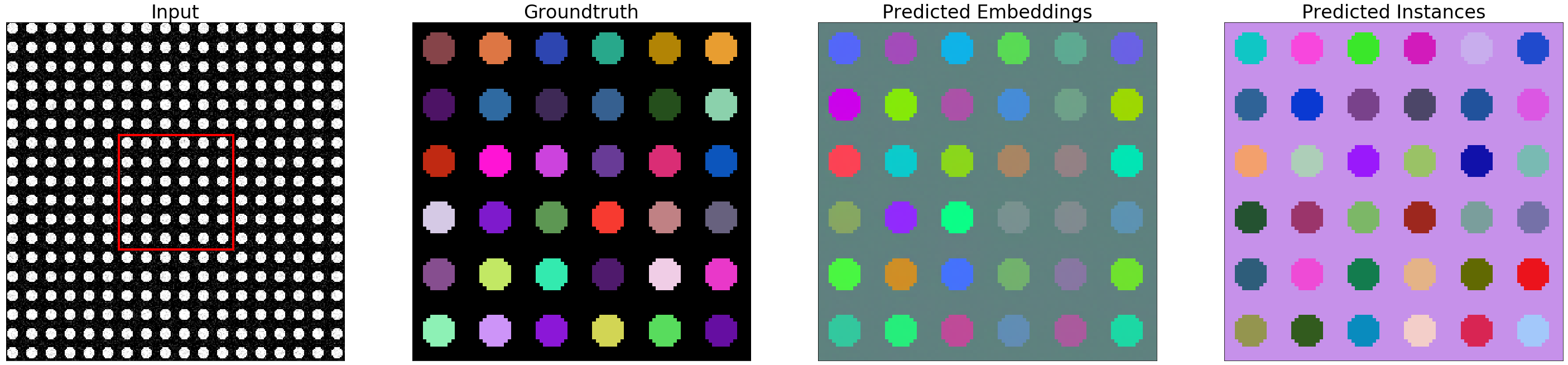}
    \caption{Slight gaussian noise added to the input: Instances distinguished despite object spacing $f^l$}
    \label{fig:noise}
    \end{subfigure}
    \begin{subfigure}[b]{0.98\textwidth}
    \includegraphics[width=\textwidth]{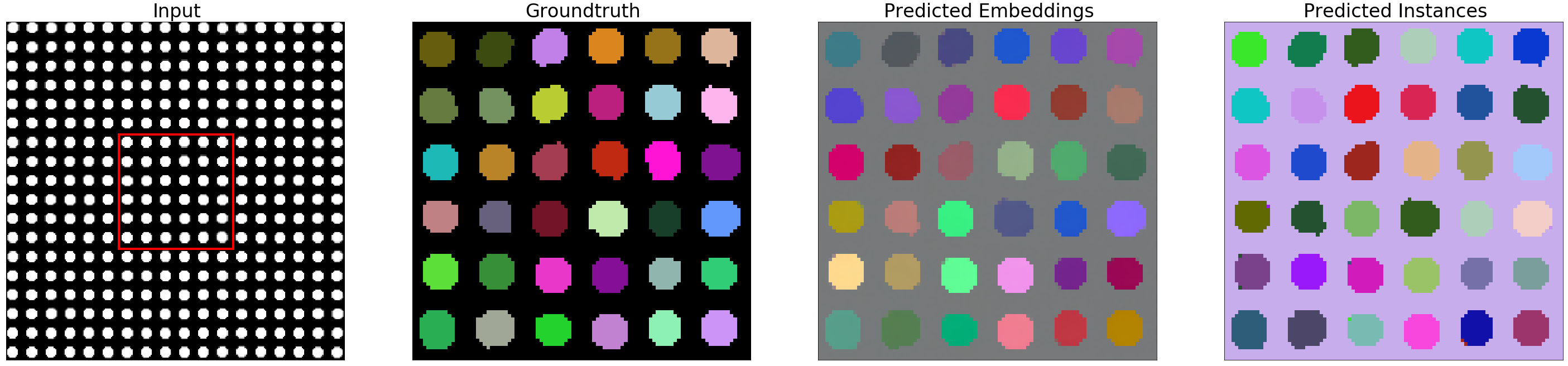}
    \caption{Elastic deformations: Instances distinguished despite object spacing $f^l$}
    \label{fig:deformations}
    \end{subfigure}
    \begin{subfigure}[b]{0.49\textwidth}
    \includegraphics[trim=1446 0 -31 0, clip,width=\linewidth]{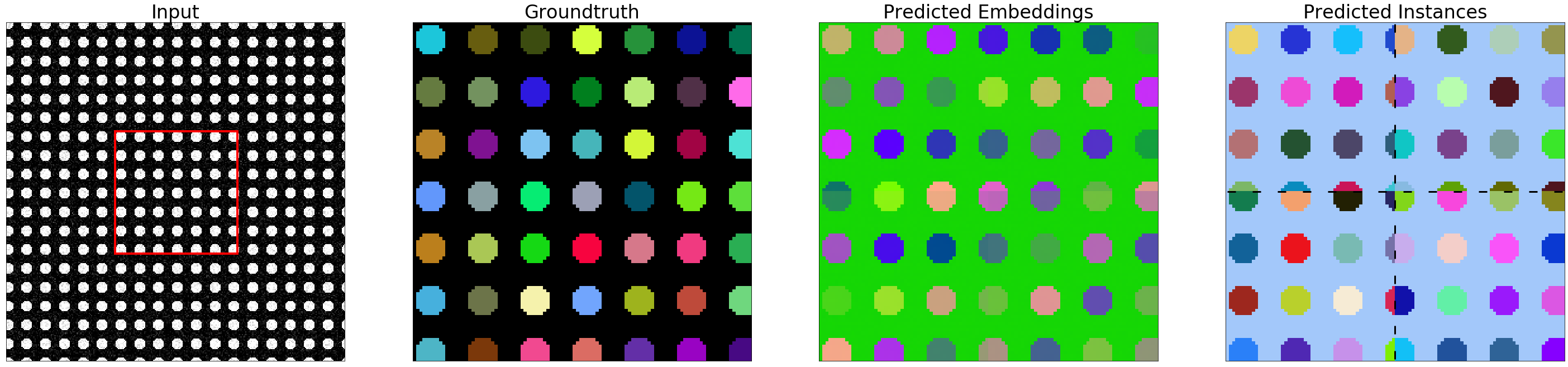}
    \caption{Slight gaussian noise: No effect on stitching issues}
    \label{fig:noise-wrong-crop}
    \end{subfigure}
    \begin{subfigure}[b]{0.49\textwidth}
    \includegraphics[trim=1446 0 -31 0, clip,width=\linewidth]{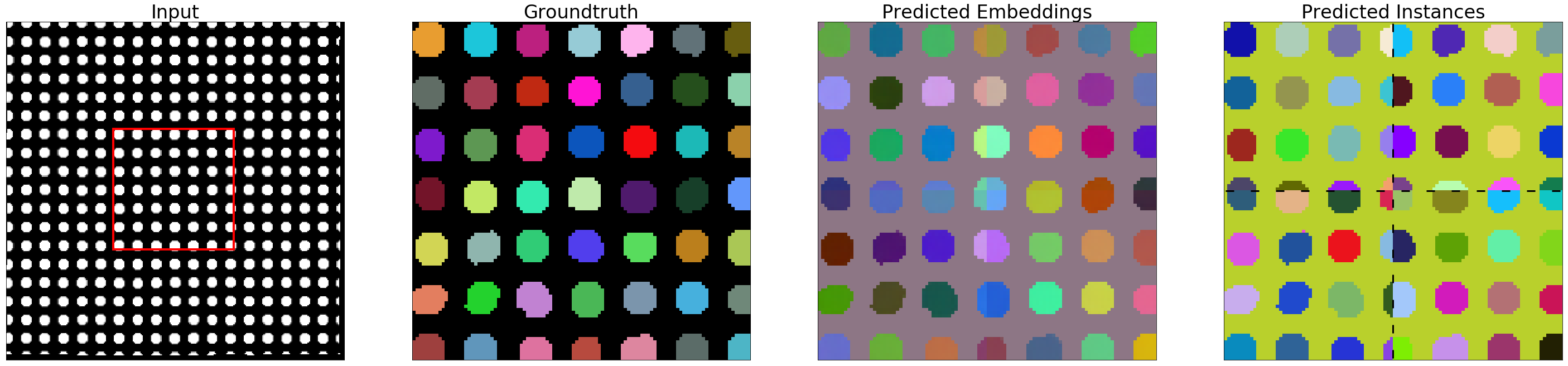}
    \caption{Elastic deformations on the input: No effect on stitching issues}
    \label{fig:deformations-wrong-crop}
    \end{subfigure}
    \caption{(a) and (b): Slight image augmentations enable a U-Net to distinguish objects that are otherwise indistinguishable due to object spacing $f^l$. Showcase: $l=4$, $f=2$, learnt upsampling. (a) Slight Gaussian noise, and (b) elastic deformations, best viewed on screen with zoom. (c) and (d): However, image augmentations do not affect the issue of inconsistencies in a tile-and-stitch approach if output tiles are not cropped to edge length $n \cdot f^l$ before stitching. Showcase:  $l=4$, $f=2$, inference output tile size $52 \neq n\cdot 2^4$.
    }
    \label{fig:noise-and-deformations}
\end{figure*}
%